\documentclass[letterpaper, 10pt, conference]{ieeeconf}  

\IEEEoverridecommandlockouts                              

\usepackage{pgfplots}
\pgfplotsset{compat=newest}
\usetikzlibrary{plotmarks}
\usepgfplotslibrary{patchplots}
\usepackage{grffile}

\usepackage{booktabs}
\usepackage{multicol}
\usepackage{multirow}
\usepackage{rotating}
\usepackage{amsmath}
\usepackage{amssymb}
\usepackage{comment}
\usepackage{color}
\usepackage{graphicx}
\usepackage{mathtools}
\usepackage{subcaption}
\usepackage[hidelinks,bookmarks=true]{hyperref}

\newtheorem{prop}{Proposition}

\DeclareMathOperator{\Bet}{Beta}
\DeclareMathOperator{\Gam}{Gamma}

\DeclareMathOperator{\Del}{Delta}

\DeclareMathOperator{\E}{E}
\DeclareMathOperator{\var}{var}

\def\atan2{\mathop{\rm atan2}}

\title{\LARGE \bf
Closed-Form Full Map Posteriors for Robot Localization \\with Lidar Sensors
}

\author{Lukas Luft, Alexander Schaefer, Tobias Schubert, Wolfram Burgard
\thanks{\copyright\ 2017 IEEE. Personal use of this material is permitted.  Permission from IEEE must be obtained for all other uses, in any current or future media, including reprinting/republishing this material for advertising or promotional purposes, creating new collective works, for resale or redistribution to servers or lists, or reuse of any copyrighted component of this work in other works.}
\thanks{Lukas Luft and Alexander Schaefer contributed equally to this work.}
\thanks{This work has been partially supported by the European Commission in the Horizon
	2020 framework program under grant agreement 644227-Flourish; by the Graduate School
	of Robotics in Freiburg; and by the State Graduate Funding Program of Baden-W\"{u}rttemberg.}
\thanks{All authors are with the Department of Computer Science, University of Freiburg, Germany.}
\thanks{\tt\small \{luft,aschaef,tobschub,burgard\}}
\thanks{\tt\small{@cs.uni-freiburg.de}}
\thanks{Digital Object Identifier (DOI): 10.1109/IROS.2017.8206583}}%

\begin{document}

\maketitle
\thispagestyle{empty}
\pagestyle{empty}

\begin{abstract}
  A popular class of lidar-based grid mapping algorithms computes for
  each map cell the probability that it reflects an incident laser
  beam. These algorithms typically determine the map as the set of reflection
  probabilities that maximizes the likelihood of the underlying laser
  data and do not compute the full posterior distribution over all
  possible maps.  Thereby, they discard crucial information about the
  confidence of the estimate. The approach presented in
  this paper preserves this information by determining the full map
  posterior.  In general, this problem is hard because distributions
  over real-valued quantities can possess infinitely many dimensions.
  However, for two state-of-the-art beam-based lidar models, our
  approach yields closed-form map posteriors that possess only two
  parameters per cell.  Even better, these posteriors come for free,
  in the sense that they use the same parameters as the traditional
  approaches, without the need for additional computations.  An
  important use case for grid maps is robot localization, which we
  formulate as Bayesian filtering based on the closed-form map
  posterior rather than based on a single map.  The resulting measurement
  likelihoods can also be expressed in closed form.  In simulations
  and extensive real-world experiments, we show that leveraging the
  full map posterior improves the localization accuracy
  compared to approaches that use the most likely map.
\end{abstract}

\section{Introduction}

Robot mapping and localization are probabilistic processes.
Therefore, it is desirable to determine the posterior probability
distribution over all possible maps given all observations rather than
to determine a particular map. For some types of grid maps, it is
well-known how to compute this distribution.

Grid maps are a popular representation of the environment of a robot.
In their basic formulation, each voxel of the map holds a binary value
which expresses whether the voxel is occupied or not.  For these
so-called occupancy grids, the posterior distribution over each
occupancy state is characterized by one real-valued parameter.
Moravec~\cite{Moravec1988:GM} and Elfes~\cite{Elfes1989:Thesis,
Elfes1989_OGM} show how to compute this parameter. 

In real-world scenarios, however, map voxels are not always completely
free or completely occupied. They often contain structures smaller
than the grid resolution.  As occupancy grids are not capable of
representing these structures, it makes sense to use real-valued
maps.

A popular example of real-valued maps are so-called reflection
maps~\cite{Haehnel2003:dynamic_environments}.  Another method which
also characterizes cells by real values builds so-called decay-rate
maps~\cite{Schaefer17:Decay}.  In contrast to posteriors over discrete
map values, posteriors over real-valued maps, like reflection maps and
decay-rate maps, can contain infinitely many parameters.  Therefore,
one typically only computes the mode of the map posterior and uses it
as map -- with few exceptions~\cite{Marks08:GammaSLAM}.

\begin{figure}
\centering
\includegraphics[width=\columnwidth]{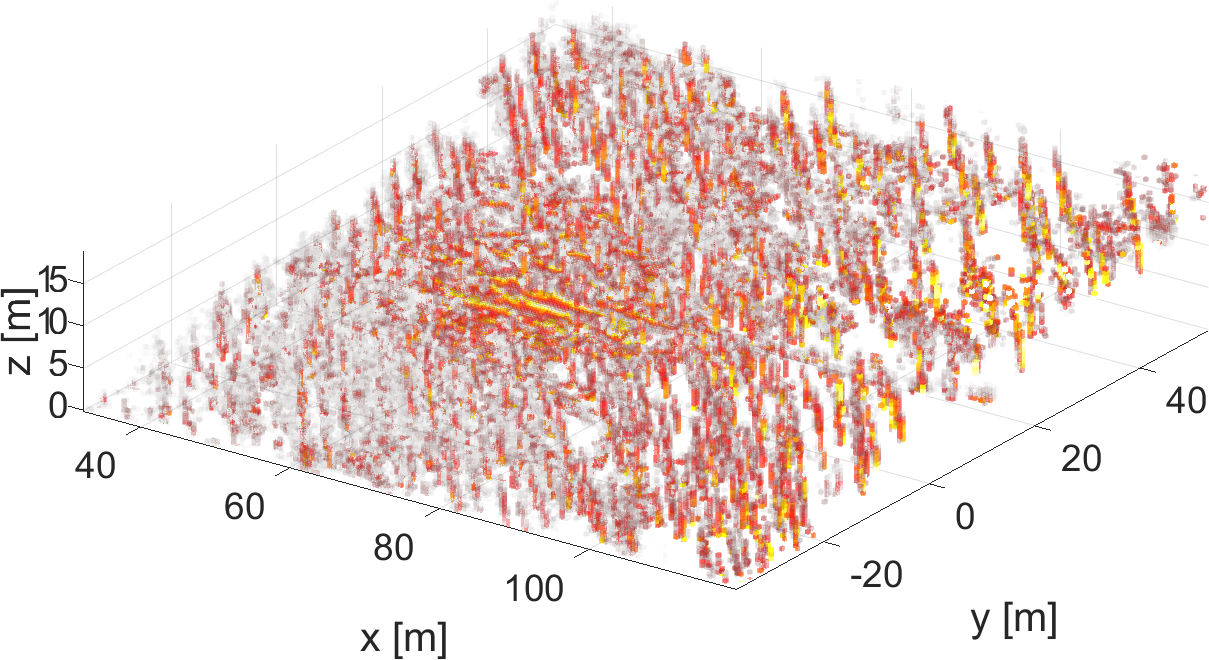}
\caption{Perspective view of a section of the reflection map built
  from the forest dataset. The map encodes the reflection probability
  of each voxel by the voxel color: Bright yellow corresponds to low
  reflection probability, dark red corresponds to high reflection
  probability. Although the map is highly cluttered, one can clearly
  recognize a large number of tree trunks.}
\label{fig:schauinsland}
\end{figure}

In this paper, we present a method to derive the full posterior over real-valued grid maps based on data provided by lidar sensors. 
Our approach, which is applicable to a broad class of forward sensor models, relies on a rigorous Bayesian formulation of the mapping process. 
For the reflection model and the decay-rate model in particular, the proposed approach leads to closed-form map posteriors.
These posteriors come for free: The most likely map already contains the required parameters.

In addition, we leverage the full map posterior for robot localization, the process of estimating the belief over the robot pose. Just like approaches that use a single given map, we can express the recursive Bayesian update in closed-form. 
Although our approach possesses the same computational complexity as the former ones, we demonstrate that it yields higher accuracies in extensive localization experiments.

\section{Related Work}
\label{sec:rw}

The proposed approach computes the posterior over real-valued grid maps.
Therefore, we structure our overview of the related work as follows:
We start with grid-based mapping approaches that compute the full posterior, move on to approaches which compute the most likely real-valued grid map, and close with posteriors over feature-based maps.

In robotics, occupancy grid maps are widely used.
To derive the posterior over their binary values, most approaches assume that the individual voxels are independent.
Then, the binary Bayes filter allows to recursively update the map posterior based upon the inverse sensor model, as shown by Moravec~\cite{Moravec1988:GM} and Elfes~\cite{Elfes1989:Thesis}, \cite{Elfes1989_OGM}. 
To obtain a full posterior over a discrete map in the context of SLAM, Doucet~et~al.~\cite{Doucet2000:RBPF} and Tipaldi~et~al.~\cite{Tipaldi13} employ a Rao-Blackwellized particle filter~\cite{Murphy2000}. 
Each of the particles represents not only a pose hypothesis, but also holds a distribution over a discrete map.
Thrun~\cite{Thrun2003:OccupancyMaps} uses a forward sensor model to compute posteriors over grid maps. 
He drops the assumption of voxel independence by accounting for measurement noise.
Marks~et~al.~\cite{Marks08:GammaSLAM} present an approach to compute posterior distributions over real-valued grid maps; in their case, each map voxel represents the height variance of the surface.

Other than Marks~et~al.~\cite{Marks08:GammaSLAM}, most approaches for real-valued grid maps compute the most likely map only: 
H{\"a}hnel~et~al.~\cite{Haehnel2003:dynamic_environments}, for example, extend the reflection model by introducing a binary variable that expresses whether a reflection is caused by a dynamic object or a static object. 
The recently introduced decay-rate model~\cite{Schaefer17:Decay} also produces real-valued grid maps.
These maps represent decay rates of the laser ray instead of reflection probabilities. 
All approaches in this paragraph have in common that they consistently leverage the forward sensor model for both mapping and localization.

Instead of using voxels, maps can also be represented by a finite set of landmarks.
Extended Kalman Filtering techniques assume the robot pose and the positions of these landmarks to be normally distributed and compute the full posterior over the robot pose and the map in closed form.
For one example among a wide range of publications in this context, see~\cite{Castellanos99:SLAM}.
Extended Kalman Filtering is also a popular choice for collaborative localization, where robots use their teammates as moving landmarks.
The belief of the joint pose state can then be interpreted as posterior over a dynamic landmark map, see for example~\cite{Luft16:DCL}.

\section{Approach}
\label{sec:approach}

This section describes how to calculate the full posterior over a real-valued
grid map and how to use it for localization.
Our approach is applicable to a broad class of beam-based sensor models, which we define in~{\ref{sec:recap}. We call them factorizing models.
In this section, we also recall the formulas for two examples of this class:
the reflection model and the decay-rate model.
In \ref{sec:recursion}, we derive a recursive update equation to compute
the posteriors over grid maps based on factorizing models.
We leverage this equation in \ref{sec:explicit_posterior} to derive closed-form posteriors over reflection maps and decay-rate maps.
Once the posteriors are established, we move on to perform robot localization:
Section~\mbox{\ref{sec:localization}} establishes a general recursive Bayesian update equation for localization based on map posteriors rather than based on a single given map.
For the reflection model and the decay-rate model, this update equation possesses closed form, as presented in~\ref{sec:explicit_likelihood}.
Table~\ref{tab:notation} provides an overview over the notation used throughout the paper.

\bgroup
\def\arraystretch{1.3}
\begin{table}[ht]
\centering
 \begin{tabular}
{ | r l | }
	\hline
  $i$ 			& voxel index\\
  $v_i$			& $i^\textrm{th}$ voxel\\ 
  $H_i$			& total number of hits in $v_i$ during mapping\\
  $M_i$			& total number of misses in $v_i$ during mapping\\
  $\lambda_i$   & decay rate in $v_i$ \\
  $\mu_i$		& reflection probability in $v_i$\\
  $x$			& sensor pose with respect to map frame\\
  $m$			& map \\
  $m_i$			& map value in voxel $v_i$\\
  $r$           & radius of a laser ray\\
  $r_i$			& distance that a ray travels inside $v_i$\\
  $R_i$			& total distance that all rays travel in $v_i$ during mapping\\
  $\mathcal{I}$ & set of all voxels\\
  $N=\left|\mathcal{I}\right|$	& number of all voxels\\
  $\mathcal{I}(r,x)$ & set of voxels entered by a beam with $r$ and $x$\\
  $X_m$ 		& sensor poses during mapping \\
  $Z$			& sensor measurements recorded during localization \\
  $Z_m$ 		& sensor measurements recorded during mapping \\
  $z$			& most recent measurement \\
  \hline
\end{tabular}
\caption{Notation}
\label{tab:notation}
\end{table}
\egroup

\subsection{Factorizing Forward Sensor Models}
\label{sec:recap}

The present paper deals with mapping and localizing based on lidar data.
The formalism presented later on is valid for a broad class of sensor models which we call factorizing models.
They are characterized by the following property:
\begin{align}
\label{eq:p_factorized}
 p(r\mid x,m) = \prod_{i\in\mathcal{I}(r,x)} f(r_i, m_i,\delta_i),
\end{align}
where $r$ is the length of the measured laser ray, $x$ denotes the sensor pose, and $m$ is a fixed map, $\mathcal{I}(r,x)$ is the set of indices of all voxels which the beam enters, $r_i$ is the radius that the beam travels within voxel $v_i$, $m_i$ denotes the map value of voxel $i$, and $\delta_i(r,x)$ tells whether or not voxel $i$ reflects the ray:
\begin{align}
 \delta_{i}=\delta_i(r,x)=\begin{cases}1& \text{if $v_i$ reflects the ray}\\
                           0 & \text{else}
                          \end{cases}
\end{align}

The reflection model~\cite{Haehnel2003:dynamic_environments} defines $f(r_i, m_i, \delta_i)$ as a binomial distribution over the event $\delta_i$:
\begin{equation}
\label{eq:mu_factors}
 f(r_i, \mu_i,\delta_i) = \mu_i^{\delta_i} (1-\mu_i)^{1-\delta_i} = f(\mu_i, \delta_i).
\end{equation}
It disregards the information about how far the ray travels inside each voxel: \eqref{eq:mu_factors} does not depend on $r_i$.

In contrast, the decay-rate model~\cite{Schaefer17:Decay} incorporates this information, as it computes $f$ as follows:
\begin{align}
\label{eq:lambda_factors}
 f(r_i, \lambda_i,\delta_i)=
\lambda_{i}^{\delta_i} e^{-\lambda_i r_i}.
\end{align}
Here, $\lambda_i$ is the decay rate within voxel $i$.
If we fix the value of $r_i$, Equation \eqref{eq:lambda_factors} yields a binary Poisson distribution over $\delta_i$. Conversely, for $\delta_i=1$, it yields an exponential distribution over $r_i$.

\subsection{Recursive Map Update}
\label{sec:recursion}

This section addresses the core of our approach. We show how to calculate a posterior distribution over all maps.
This distribution is called the belief \mbox{$bel(m):=p(m\mid Z_m, X_m)$},
where $Z_m$ denotes the set of all measurements recorded during the mapping process, and where $X_m$ denotes the set of corresponding sensor poses.
In order to calculate the full map posterior, we first need to introduce the following definitions.
\begin{align}
 {bel}(m)&=p(m\mid{Z}_m,X_m)\\
 \overline{bel}(m)&=p(m\mid\overline{Z}_m,X_m)\\
 {bel}(m_i)&=p(m_i\mid{Z}_m,X_m) \label{eq:belmi}\\
 \overline{bel}(m_i)&=p(m_i\mid\overline{Z}_m,X_m) \label{eq:belmibar}
\end{align}
Here, $\overline{Z}_m=Z_m\setminus \{z\}$ represents the set of mapping measurements without the most recent measurement $z$. Note that when referring to general sensor models, we call the sensor output~$z$; only in the context of factorizing models, we write $r$ for radius.
Proposition~\ref{prop:update} now shows how to recursively compute the full map posterior from the latest belief and the latest measurement.

\begin{prop}
\label{prop:update}
Assuming a factorizing sensor model
\begin{align}
 \label{eq:prop_factorizing}
 p(r\mid x,m) =& \prod_{i\in\mathcal{I}(r,x)} f(r_i, m_i,\delta_i),
\end{align}
and mutual independence of the individual voxels
\begin{align}
 \label{eq:prop_b}
bel(m)	=&	\prod_{i=1}^N bel(m_i),
\end{align}
the belief over each map value is recursively updated according to
\begin{align}
\label{eq:b_update}
bel(m_i) = \eta_i\ \overline{bel}(m_i)\ f(r_i, m_i,\delta_i),
\end{align}
where $\eta_i$ is a normalizing constant independent of $m_i$.
\end{prop}

\begin{proof}
To prove the above proposition, we define the following notation:
\begin{align}
\label{eq:n1}
&m_{\mathcal{I}\setminus i}:=m\setminus \{m_i\}
\end{align}
and
\begin{gather*}
\int\displaylimits_{ m_{\mathcal{I}\setminus i} }(\cdot)\,dm_{\mathcal{I}\setminus i} := \\
\int\displaylimits_{ m_1 }\dots \int\displaylimits_{ m_{i-1} } \int\displaylimits_{ m_{i+1} }\dots\int\displaylimits_{ m_N } (\cdot)\, dm_{N}\dots dm_{i+1} dm_{i-1}\dots dm_{1}.
\end{gather*}
Using this notation, we derive Proposition~\ref{prop:update} as follows:
\begin{align}
bel(m_i) 
&=	p(m_i\mid Z_m,X_m) \nonumber \\
&\stackrel{\eqref{eq:belmibar}}{=}\eta\, p(z\mid m_i, \overline{Z}_m,X_m)\ \overline{bel}(m_i) \nonumber \\
&=\eta\	\overline{bel}(m_i) \int\displaylimits_{ m_{\mathcal{I}\setminus i} } p(z, m_{\mathcal{I}\setminus i}\mid m_{i}, \overline{Z}_m,X_m)\ dm_{\mathcal{I}\setminus i} \nonumber \\
&\stackrel{\eqref{eq:n1}}{=}\eta\,	\overline{bel}(m_i) \int\displaylimits_{ m_{\mathcal{I}\setminus i}} p(z\mid m, \overline{Z}_m,X_m) \nonumber \\
&\hphantom{=\eta\ \overline{bel}(m_i) \int\displaylimits_{ m_{\mathcal{I}\setminus i}}}  p(m_{\mathcal{I}\setminus i} \mid m_{i}, \overline{Z}_m,X_m)\ dm_{\mathcal{I}\setminus i} \nonumber \\
&= \eta\ \overline{bel}(m_i) \int\displaylimits_{ m_{\mathcal{I}\setminus i}} \prod_{j\in \mathcal{I}(r,x)}f(r_j, m_j,\delta_j) \nonumber \\
&\hphantom{= \eta\,\overline{bel}(m_i) \int\displaylimits_{ m_{\mathcal{I}\setminus i}}} \prod_{k\in \mathcal{I}\setminus \{i\}} \overline{bel}(m_k)\ dm_{\mathcal{I}\setminus i} \label{eq:factor} \\
&= \eta_i\ \overline{bel}(m_i)\ f(r_i, m_i,\delta_i). \nonumber 
\end{align}
$\eta$ and $\eta_i$ are normalizing constants independent of $m_i$.
To obtain \eqref{eq:factor}, we make use of both \eqref{eq:prop_factorizing}  and \eqref{eq:prop_b}.
To transition from \eqref{eq:factor} to the last line, we pull $f(r_i, m_i,\delta_i)$ out of the integral and merge the remaining integral, which is independent of $m_i$, with the normalizer.
\end{proof}


The update equation~\eqref{eq:b_update} in Proposition~\ref{prop:update} might look familiar:
It is a generalization of the well-known map update
\mbox{$bel(m) = \eta\ \overline{bel}(m)\ p(z\mid x,m)$}, which can be derived from Bayes rule in a straight-forward manner,
see equation (7) in~\cite{Thrun2003:Survey}.
Another update equation which is related to~\eqref{eq:b_update} is the voxel-wise update for binary occupancy maps, see (18) in~\cite{Thrun2003:Survey}. In contrast to the proposed update equation, the latter employs the inverse sensor model.

\subsection{Closed-Form Map Posteriors}
\label{sec:explicit_posterior}

In this section, we leverage Proposition~\ref{prop:update} to derive the closed-form map posteriors for the reflection model and for the decay-rate model.
For the reflection model, update equation \eqref{eq:b_update} yields the following posterior over $\mu_i$:
\begin{align}
bel(\mu_i)
& \propto \prod_{Z_m} f(\mu_i,\delta_i)\ p(\mu_i) \notag \\
& \stackrel{\eqref{eq:mu_factors}}{\propto} \mu_i^{H_i} (1-\mu_i)^{M_i}\ p(\mu_i) \notag \\
& \propto \Bet(H_i+1,M_i+1)\ p(\mu_i).
\label{eq:bel_mu}
\end{align}
Here, $p(\mu_i)$ denotes the prior distribution over the map, $H_i$ tells how many rays are reflected in $v_i$, and $M_i$ is the number of rays that penetrate $v_i$ without reflection.
$\Bet(\cdot)$ denotes a beta distribution. 

If the prior is a beta distribution $p(\mu_i)=\Bet(\alpha,\beta)$, which is the conjugate prior for the binomial distribution $f(\mu_i,\delta_i)$, Equation~\eqref{eq:bel_mu} yields
\begin{align}
\label{eq:b(mu)}
bel(\mu_i)
=&\Bet(H_i+\alpha,{M}_i+\beta).
\end{align}

The most likely reflection map can easily be derived from~\eqref{eq:b(mu)}:
Assuming a uniform prior distribution \mbox{$p(\mu_i)=\Bet(1,1)=1$} and computing the mode of the resulting beta posterior distribution yields the same result as formulated by H{\"a}hnel~et~al.~\cite{Haehnel2003:dynamic_environments} for maximum likelihood reflection maps:
\begin{align}
\label{eq:mu}
 \mu_i^*=\frac{H_i}{H_i+M_i}.
\end{align}

For the decay-rate model, the update equation \eqref{eq:b_update} becomes
\begin{align*}
\nonumber
bel(\lambda_i)
& \propto \prod_{Z_m} f(r_i, \lambda_i,\delta_i)\ p(\lambda_i) \\
& \stackrel{\eqref{eq:lambda_factors}}{\propto} 
  \lambda_i^{H_i} e^{-\lambda_i R _i}\ p(\lambda_i) \nonumber \\
& \propto \Gam(H_i+1,R_i)\ p(\lambda_i).
\end{align*}
\mbox{$\Gam(\cdot)$} denotes a gamma distribution, $p(\lambda_i)$ is the prior map distribution, and $R_i$ is the sum of the distances all rays travel within $v_i$.

If $p(\lambda_i)$ is a gamma distribution $\Gam(\alpha,\beta)$, which is the conjugate prior for the Poisson distribution and for the exponential distribution, we obtain the gamma-distributed belief
\footnote{In the context of height maps, Marks~et~al.~\cite{Marks08:GammaSLAM} also obtain gamma-shaped posteriors over grid values.}
\begin{align}
\label{eq:b(lambda)}
 bel(\lambda_i)=\Gam\left(H_i+\alpha,R_i+\beta \right).
\end{align}

Setting $\alpha=1$ and $\beta=0$ leads to the so-called uninformative prior.
Plugging this prior into~\eqref{eq:b(lambda)} and computing the mode of the resulting posterior leads to the decay rates of the maximum likelihood approach as given in~\cite{Schaefer17:Decay}:
\begin{align}
\label{eq:lambda}
 \lambda_i^* = \frac{H_i}{R_i}.
\end{align}

For both the reflection model and the decay-rate model, the parameters $\alpha$ and $\beta$ of the prior distribution need to be estimated during mapping. 
In Section~\ref{sec:exp}, we explain how we obtain the prior parameters used throughout the experiments.

\subsection{Localization with Map Posteriors}
\label{sec:localization}

In this section, we formulate robot localization on the basis of the full posterior over the map rather than on the basis of a fixed map. 
In contrast to the well-known approach~\cite{Burgard2005:Probabilistic_Robotics} which computes the belief over the robot pose on the basis of the given map $m$ as
\begin{align}
\label{eq:standard_update}
bel_m(x) = \eta\ \overline{bel}(x)\ p(z\mid x,m),
\end{align}
we leverage the full posterior \mbox{$bel(m)=p(m\mid X_m,Z_m)$} instead of $m$. 
Consequently, we derive the pose belief as follows:
\begin{align}
bel(x)
& =	p(x\mid Z,Z_m,X_m) \notag \\
& = \eta\ p(z\mid x,\overline{Z},Z_m,X_m)\ 
    \underbrace{p(x\mid \overline{Z},Z_m,X_m)}_{=:\overline{bel}(x)} \notag \\
& = \eta\ \overline{bel}(x) 
    \int p(z\mid x,m)\ p(m\mid \overline{Z},Z_m,X_m)\ dm \notag \\
& = \eta\ \overline{bel}(x)\underbrace{\int p(z\mid x,m)\ bel(m)\ dm}_{=:L(z,x)}.
  \label{eq:measurement_likelihood}
\end{align}
Here, $X_m$ and $Z_m$ are the poses and measurements recorded during
the mapping process, respectively, $Z$ denotes the measurements recorded during localization,
and $\overline{Z}=Z\setminus \{z\}$ is the set of all measurements but the most recent one.

Equations~\eqref{eq:standard_update} and \eqref{eq:measurement_likelihood} differ only in the last term, which, in~\eqref{eq:standard_update}, is called the measurement likelihood. 
Analogously, we call~$L(z,x)$ in~\eqref{eq:measurement_likelihood} the measurement likelihood based on the map posterior.
The next section presents closed-form solutions of~$L(z,x)$ for the reflection model and the decay-rate model.

\subsection{Closed-Form Measurement Likelihoods}
\label{sec:explicit_likelihood}

For two particular factorizing sensor models, the reflection model and the decay-rate model, the solution of the integral contained in the measurement likelihood $L(z,x)$ leads to closed-form expressions.
To compute those, we need to factorize $L(z,x)$ first:
\begin{align*}
&L(z,x)\\
=& \int\displaylimits_m p(z\mid x,m)\ bel(m)\ dm \\
=& \int\displaylimits_m \prod_{i\in \mathcal{I}(r,x)} f(r_i, m_i,\delta_i) 
   \prod^N_{j=1} bel(m_j)\ dm \\
=& \int\displaylimits_{m_{\mathcal{I}(r,x)}}  \prod_{i\in \mathcal{I}(r,x)}    
   f(r_i, m_i,\delta_i) \prod_{j\in \mathcal{I}(r,x)} bel(m_j)\ dm_{\mathcal{I}(r,x)} \\
=& \prod_{i\in \mathcal{I}(r,x)} l(r_i,\delta_i)
\end{align*}
with 
\begin{align}
\label{eq:l_general}
l(r_i,\delta_i) := \int  f(r_i, m_i,\delta_i)\ bel(m_i)\ dm_{i}.
\end{align}

Now, we evaluate this integral for both sensor models.
For the reflection model with beta-shaped prior \mbox{$p(\mu_i)=\Bet(\alpha,\beta)$}, we
take the distribution $f(\mu_i,\delta_i)$ from \eqref{eq:mu_factors} and the posterior from \eqref{eq:b(mu)}, plug them into \eqref{eq:l_general}, and solve the integral:
\begin{align}
\label{eq:lref}
l_\text{ref}(r_i,\delta_i) 
= l_\text{ref}(\delta_i)
= \frac{\left(H_i+\alpha\right)^{\delta_i}\left(M_i+\beta\right)^{1-\delta_i}}
  {H_i+\alpha+{M}_i+\beta}.
\end{align}

If the posterior $bel(\mu_i)$ is set to a Dirac delta distribution \mbox{$\Del(\mu_i-\mu_i^*)$},
Equation~\eqref{eq:l_general} reproduces the measurement likelihood based on the most likely map, as derived in~\cite{Haehnel2003:dynamic_environments}:
\begin{align}
\label{eq:lref_det}
l_\text{ref}(\delta_i) 
= f\left(\delta_i\right) 
= \frac{H_i^{\delta_i}M_i^{1-\delta_i}}{H_i+{M}_i}.
\end{align}
Equation~\eqref{eq:lref_det} is only valid for $H_i+M_i > 0$. 
If no ray has visited the voxel during mapping, the maximum likelihood approach has to assign some initial value.
In contrast, \eqref{eq:lref} is also valid for voxels that have not been visited by any ray during mapping.

Note that in the special case of $\alpha=\beta$, Equation~\eqref{eq:lref_det} can be transformed into \eqref{eq:lref} using Laplace smoothing --
see for example Equation~(13) in \cite{Valcarce16:Laplace_smoothing}.

In order to obtain closed-form solutions for the decay-rate model, we assume a gamma-shaped prior \mbox{$p(\lambda_i)=\Gam(\alpha,\beta)$} and plug
\eqref{eq:lambda_factors} and \eqref{eq:b(lambda)} into \eqref{eq:l_general}. This leads to
\begin{align}
\nonumber
l_\text{dec}(r_i,\delta_i)  
= \left(\frac{R_i+\beta}{R_i+\beta+r_i} \right)^{H_i+\alpha}   
  \left(\frac{H_i+\alpha}{R_i+\beta+r_i}\right)^{\delta_i}
\end{align}

Analogously to the reflection model, the measurement likelihood based on the most likely map as derived in~\cite{Schaefer17:Decay} can be reproduced from \eqref{eq:l_general} by setting \mbox{$bel(\lambda_i)=\Del(\lambda_i-\lambda_i^*)$}:
\begin{align*}
l_\text{dec}(r_i,\delta_i)
= f\left(r_i, \delta_i\right)
= e^{-\frac{H_i}{R_i} r_i}\left(\frac{H_i}{R_i}\right)^{\delta_i}.
\end{align*}

Until now, we have assumed that the lidar sensor reports a reflection for each emitted beam.
In practice, however, the sensor range is always limited by a lower bound $r_\text{min}$ and an upper bound $r_\text{max}$.
The following equations show for both measurement models how to attribute probabilities to these out-of-range measurements:
\begin{align*}
P(r< r_{min}\mid x,Z_m,X_m)
&= 1 - \prod_{i\in \mathcal{I}(r_\text{min},x)} l(r_i,\delta_i=0), \\
P(r> r_{max}\mid x,Z_m,X_m) 
&= \prod_{i\in \mathcal{I}(r_\text{max},x)} l(r_i,\delta_i=0).
\end{align*}

\section{Experiments}
\label{sec:exp}

In the previous section, we have derived all the necessary equations to compute the posterior over a grid map and to leverage this posterior for localization. 
Now, to demonstrate that localization benefits from employing the full map posterior instead of the most likely map, we perform a simulation and extensive real-world experiments.

As shown in Section~\ref{sec:explicit_posterior}, the computation of the map posterior  for the reflection model and for the decay-rate model requires the specification of the two parameters $\alpha$ and $\beta$ of the prior distribution. 
In our experiments, we determine them as follows.
First, we compute the empirical mean and variance of the map value over all voxels.
Then, we choose $\alpha$ and $\beta$ such that the mean and variance of the 
parameterized distributions equal these empirical values.
To that end, we solve the known equations for the mean and the variance of the beta and gamma distribution for $\alpha$ and $\beta$. For the reflection model, we obtain
\begin{align*}
\alpha &= -\frac{
		\E\left[\mu\right]
		\left(
			\E\left[\mu\right]^2 - \E\left[\mu\right] + \var\left[\mu\right]
		\right)
	}
	{\var\left[\mu\right]}, \\ 
\beta  &= \frac{
		\E\left[\mu\right] - \var\left[\mu\right] + \E\left[\mu\right]\var\left[\mu\right] - 2\E\left[\mu\right]^2 + \E\left[\mu\right]^3
	}
	{\var\left[\mu\right]}.
\end{align*}
For the decay-rate model, the parameters are
\begin{align*}
\alpha &= \frac{\E\left[\lambda\right]^2}{\var\left[\lambda\right]}, \\
\beta  &= \frac{\E\left[\lambda\right]}{\var\left[\lambda\right]}.
\end{align*}
To strictly avoid the repeated use of the same information, one has to estimate these parameters for each voxel separately, computing the empirical mean and variance over all voxels but the considered one.
However, the sheer number of observations in our datasets render this effect negligible.
Hence, it is sufficient to compute $\alpha$ and $\beta$ only once.

Throughout the experiments section, MLM denotes the maximum likelihood approach, while FMP refers to the proposed approach, which leverages the full map posterior.
REF and DEC denote the reflection model and the decay-rate model, respectively.

\subsection{Localization in Simulation}
\label{sec:sim}

We simulate the following scenario to evaluate the localization performance with both maximum likelihood maps and full posteriors:
A mobile robot is located in a corridor that is modeled by a row of $N=100$ consecutive voxels.
Each experiment run consists of two phases:
During the mapping phase, the robot visits every voxel $n$ times and for each voxel collects the measurements $H_i$, $M_i$, and $R_i$, which it uses to build both a reflection map and a decay-rate map.
In the localization phase, which consists of 100 iterations, the robot traverses the corridor from start to end knowing its motor commands.
In each iteration, the robot moves to the next voxel, fires its sensor once, and updates the belief over its pose.
Therefor, with the maximum likelihood approach it uses~\eqref{eq:standard_update}, with the proposed approach it uses~\eqref{eq:measurement_likelihood}.
After the robot has arrived at the last voxel of the corridor, we evaluate the pose belief at the true pose averaged over each iteration:
\begin{align*}
\rho := \mathbf{E}\left(bel\left(x_\text{true}\right)\right).
\end{align*}

We perform $10,000$ such localization runs for each \mbox{$n\in\{1,2,3,4,5,10,20,50,100,200\}$}.

In each run, we synthesize a new map $m$ by drawing samples from a uniform distribution for the reflection model or from a gamma distribution $\Gam\left(1,1\right)$ for the decay-rate model.
This map is hidden from the localization algorithms.
We use $m$ to simulate the measurements the robot records during mapping and localization by sampling from the distribution described in Equation~\eqref{eq:mu_factors} for the reflection model or from \eqref{eq:lambda_factors} for the decay-rate model.
The belief over the robot pose is initialized with a uniform distribution over $x$.

For both measurement models, we compare three algorithms:
MLM, FMP with uniform prior, and FMP with conjugate prior.
The results in Figure~\ref{fig:sim} indicate that the proposed method -- FMP with conjugate prior -- yields higher localization accuracy than the compared methods for both measurement models.
Moreover, we observe that the benefit of the proposed method is greater for the reflection model than for the decay-rate model.
We perform a one-tailed, paired-sample t-test, which validates these observations:
For the reflection model, the proposed method outperforms the other two approaches for $n\leq100$ with a probability greater than $0.9999$. For the decay-rate model, we obtain the same significance level for $n\leq4$.

\begin{figure*}[t!]
    \centering
    \begin{subfigure}[c]{.5\textwidth}
        \centering	    
        \resizebox{\textwidth}{!}{\small
%
%
\definecolor{mycolor1}{rgb}{0.44157,0.74902,0.43216}%
\begin{tikzpicture}

\begin{axis}[%
width=0.951\textwidth,
height=0.665\textwidth,
at={(0\textwidth,0\textwidth)},
scale only axis,
xmin=-0.2,
xmax=5.49831736654804,
xlabel style={font=\color{white!15!black}},
xlabel={$\ln\left(n\right)$},
ymin=0,
ymax=0.9,
ylabel style={font=\color{white!15!black}},
ylabel={$\rho$},
axis background/.style={fill=white},
axis x line*=bottom,
axis y line*=left,
legend style={at={(0.97,0.03)}, anchor=south east, legend cell align=left, align=left, draw=white!15!black}
]
\addplot [color=blue, line width=1.0pt]
 plot [error bars/.cd, y dir = both, y explicit, error bar style={solid}] 
 table[row sep=crcr, y error plus index=2, y error minus index=3]{%
0	0.429336566702808	0.075384355220395	0.075384355220395\\
0.693147180559945	0.568094150031899	0.0545471518426513	0.0545471518426513\\
1.09861228866811	0.621599118337337	0.0431269773177616	0.0431269773177616\\
1.38629436111989	0.647255951072116	0.0368946845432782	0.0368946845432782\\
1.6094379124341	0.668175863743393	0.0322047769588088	0.0322047769588088\\
2.30258509299405	0.716359207027577	0.0208671388913414	0.0208671388913414\\
2.99573227355399	0.747350425038223	0.0132820342857407	0.0132820342857407\\
3.91202300542815	0.758356691110406	0.010406011030274	0.010406011030274\\
4.60517018598809	0.765735890101999	0.00931558694732143	0.00931558694732143\\
5.29831736654804	0.763970286476162	0.00926612060120593	0.00926612060120593\\
};
\addlegendentry{FMP with beta prior}

\addplot [color=mycolor1, dashdotted, line width=1.0pt]
 plot [error bars/.cd, y dir = both, y explicit, error bar style={solid}]
 table[row sep=crcr, y error plus index=2, y error minus index=3]{%
0	0.317949055507519	0.0411209322688187	0.0411209322688187\\
0.693147180559945	0.498638985096498	0.0396757408807646	0.0396757408807646\\
1.09861228866811	0.580644267327148	0.031302504560164	0.031302504560164\\
1.38629436111989	0.623251776157871	0.026366451703955	0.026366451703955\\
1.6094379124341	0.651827163751004	0.0223648782461834	0.0223648782461834\\
2.30258509299405	0.706909678074074	0.0160141399184694	0.0160141399184694\\
2.99573227355399	0.739891236496165	0.0114517050466844	0.0114517050466844\\
3.91202300542815	0.754634714153797	0.0101122909421138	0.0101122909421138\\
4.60517018598809	0.764311812864023	0.00921959364120483	0.00921959364120483\\
5.29831736654804	0.763508370453491	0.00923297768763936	0.00923297768763936\\
};
\addlegendentry{FMP with uniform prior}

\addplot [color=red, dashed, line width=1.0pt]
 plot [error bars/.cd, y dir = both, y explicit, error bar style={solid}]
 table[row sep=crcr, y error plus index=2, y error minus index=3]{%
0	0.0391539745216354	0.00118765573252883	0.00118765573252883\\
0.693147180559945	0.0832870326579564	0.00559998032263228	0.00559998032263228\\
1.09861228866811	0.139339527836842	0.0141442320800584	0.0141442320800584\\
1.38629436111989	0.198129938177884	0.0256199292866992	0.0256199292866992\\
1.6094379124341	0.256454281859673	0.0382971058634144	0.0382971058634144\\
2.30258509299405	0.461424559108967	0.0708927432303436	0.0708927432303436\\
2.99573227355399	0.62117549702656	0.069494113879125	0.069494113879125\\
3.91202300542815	0.730237182125877	0.0281034333784785	0.0281034333784785\\
4.60517018598809	0.759919365185051	0.0140562836905504	0.0140562836905504\\
5.29831736654804	0.76283648504426	0.0106529939139553	0.0106529939139553\\
};
\addlegendentry{MLM}

\end{axis}
\end{tikzpicture}%
}
        \caption{Reflection model}
    \end{subfigure}%
    \begin{subfigure}[c]{.5\textwidth}
        \centering
        \resizebox{\textwidth}{!}{\small
%
%
\definecolor{mycolor1}{rgb}{0.44157,0.74902,0.43216}%
\begin{tikzpicture}

\begin{axis}[%
width=0.951\textwidth,
height=0.665\textwidth,
at={(0\textwidth,0\textwidth)},
scale only axis,
xmin=-0.2,
xmax=5.49831736654804,
xlabel style={font=\color{white!15!black}},
xlabel={$\ln\left(n\right)$},
ymin=0,
ymax=0.9,
ylabel style={font=\color{white!15!black}},
ylabel={$\rho$},
axis background/.style={fill=white},
axis x line*=bottom,
axis y line*=left,
legend style={at={(0.97,0.03)}, anchor=south east, legend cell align=left, align=left, draw=white!15!black}
]
\addplot [color=blue, line width=1.0pt]
 plot [error bars/.cd, y dir = both, y explicit, error bar style={solid}]
 table[row sep=crcr, y error plus index=2, y error minus index=3]{%
0	0.722443398962342	0.024344657094645	0.024344657094645\\
0.693147180559945	0.819609801729155	0.00858407817465493	0.00858407817465493\\
1.09861228866811	0.84393441445562	0.0058392162775352	0.0058392162775352\\
1.38629436111989	0.853167684791122	0.00496717005341647	0.00496717005341647\\
1.6094379124341	0.85982995095898	0.0044106545924267	0.0044106545924267\\
2.30258509299405	0.873279376872576	0.00341621031027147	0.00341621031027147\\
2.99573227355399	0.879957855714679	0.00289298268094043	0.00289298268094043\\
3.91202300542815	0.88420134364848	0.00267390746992	0.00267390746992\\
4.60517018598809	0.884926875027556	0.00268436794031937	0.00268436794031937\\
5.29831736654804	0.884579775342673	0.00270993228456282	0.00270993228456282\\
};
\addlegendentry{FMP with gamma prior}

\addplot [color=mycolor1, dashdotted, line width=1.0pt]
 plot [error bars/.cd, y dir = both, y explicit, error bar style={solid}]
 table[row sep=crcr, y error plus index=2, y error minus index=3]{%
0	0.651031450811996	0.0391879103210245	0.0391879103210245\\
0.693147180559945	0.80085167315398	0.0132136865224591	0.0132136865224591\\
1.09861228866811	0.837874509727629	0.00765001123121396	0.00765001123121396\\
1.38629436111989	0.851656386290506	0.00583305497782989	0.00583305497782989\\
1.6094379124341	0.859695520857544	0.00499242976538517	0.00499242976538517\\
2.30258509299405	0.874656650025602	0.00349926280804644	0.00349926280804644\\
2.99573227355399	0.880991989927562	0.00290378268355636	0.00290378268355636\\
3.91202300542815	0.884721533915026	0.00266424508359209	0.00266424508359209\\
4.60517018598809	0.885191221390671	0.00268154952399791	0.00268154952399791\\
5.29831736654804	0.884714206671178	0.00270886554762108	0.00270886554762108\\
};
\addlegendentry{FMP with uninformative prior}

\addplot [color=red, dashed, line width=1.0pt]
 plot [error bars/.cd, y dir = both, y explicit, error bar style={solid}]
 table[row sep=crcr, y error plus index=2, y error minus index=3]{%
0	0.44381211561625	0.0873984709942144	0.0873984709942144\\
0.693147180559945	0.76610075643959	0.0287291837595864	0.0287291837595864\\
1.09861228866811	0.83175596356292	0.0112484657233275	0.0112484657233275\\
1.38629436111989	0.851036062653771	0.00728271276719769	0.00728271276719769\\
1.6094379124341	0.861022012637767	0.00572963612920014	0.00572963612920014\\
2.30258509299405	0.876498053379717	0.00360366282859611	0.00360366282859611\\
2.99573227355399	0.882122663382338	0.00292521191381918	0.00292521191381918\\
3.91202300542815	0.885215984346692	0.0026684522600154	0.0026684522600154\\
4.60517018598809	0.885442481701114	0.00268220046027457	0.00268220046027457\\
5.29831736654804	0.884842687289185	0.00270986269792257	0.00270986269792257\\
};
\addlegendentry{MLM}

\end{axis}
\end{tikzpicture}%
}
        \caption{Decay-rate model}
    \end{subfigure}
    \caption{Localization accuracy in a simulated environment.
    $n$ is the number of observations per voxel during the mapping process.
    The data points correspond to $n\in\{1,2,3,4,5,10,20,50,100,200\}$. 
    The accuracy measure~$\rho$ shows the average probability which the algorithm assigns to the ground truth position.
    The error bars represent the variances over $10,000$ runs.}
    \label{fig:sim}
\end{figure*}
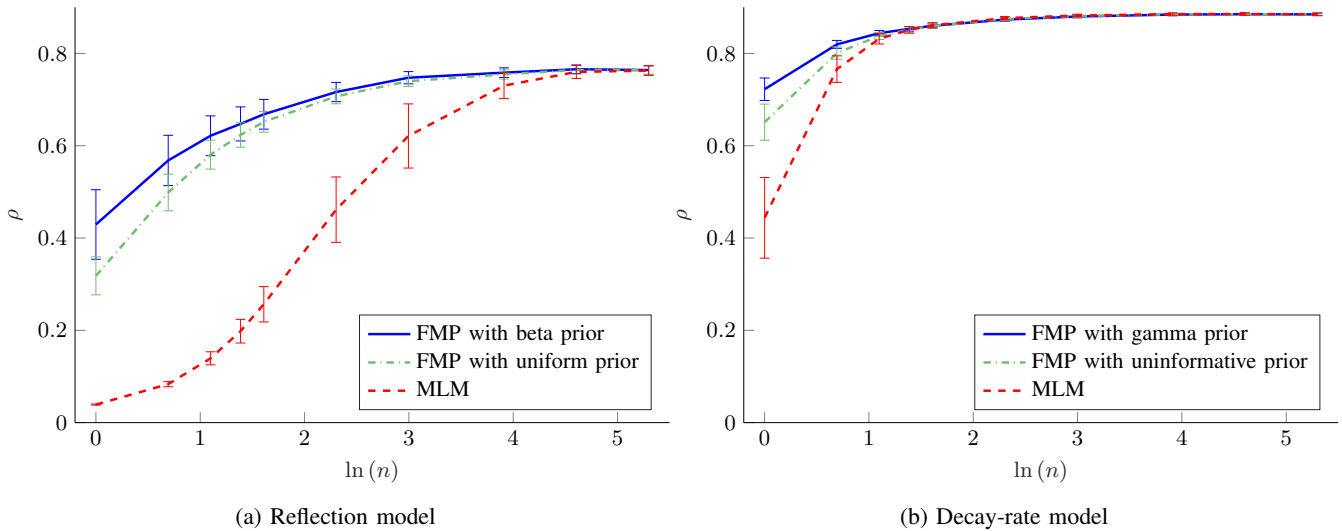

\subsection{Real-World Localization}

In order to validate the findings of the simulation in the real world, we test our approach on three datasets in which our mobile off-road robot navigates different kinds of environments: 
In the campus dataset, the robot drives around on the campus of the University of Freiburg. 
The forest dataset, a section of which is shown in Figure~\ref{fig:schauinsland}, was recorded on a small trail in the middle of a forest,
while the park dataset contains a long trajectory along a broad road through an open forest.
We tesselate the environment into cubic axis-aligned voxels with edge length 0.5\,m.
Consequently, the campus maps contain $444\times 406\times 43$ voxels, the park maps contain $515\times 561\times 41$ voxels, and the forest maps contain $393\times 403\times 86$ voxels.

We use the following hardware: 
The off-road robot VIONA by Robot Makers carries a Velodyne HDL-64E lidar sensor and an Applanix POS LV localization system,
which provides a centimeter-accurate estimate of the robot pose by fusing the data from multiple GPS sensors, an IMU, and odometry.
Due to its accuracy, we use the Applanix data as pose ground truth during mapping and in the evaluation of all localization estimates.

To assess the localization performance using full posteriors versus maximum-likelihood maps, we employ different measures:

\subsubsection{Measurement Likelihood}
We employ the logarithm of the measurement likelihood at the true pose as a measure of how well the approaches predict the measurements given the true pose.
For MLM, the measurement likelihood is defined as $p(z \mid x, m)$. For FMP, it is defined by $L(z,x)$ in \eqref{eq:measurement_likelihood}.
We sum up the likelihood values of all measurements in a dataset and divide the result for MLM by the result for FMP.
The ratios are presented in Table~\ref{tab:pz}.
For both measurement models and all three datasets, these values are greater than one, which means that FMP achieves better prediction of the measurements than MLM.

\begin{table}	
	\normalsize
	\centering
    \begin{tabular}{ l | r  r }
    		& REF & DEC \\ 
    		\midrule
    campus 	& 1.21 &   1.16\\
    forest 	& 1.22 &  1.31\\
    park 	& 1.25 &  1.27\\
    \end{tabular}
	\caption{Log-likelihood ratios of the measurements given the ground-truth pose.
	For each dataset, we show the ratio of the cumulated measurement log-likelihoods for MLM to the ones for FMP. Values greater than one mean that FMP attributes higher likelihoods to measurements given the true robot pose.}
	\label{tab:pz}
\end{table}

\subsubsection{Kullback-Leibler Divergence}

As a measure of dissimilarity, we employ the Kullback-Leibler (KL) divergence from the estimated position distribution $p$ to the ground-truth pose distribution $p_{gt}$:
\begin{align}
D_{KL}\left(p_{gt}(x) \parallel p(x)\right) = \int p_{gt}(x)\ \log\left(\frac{p_{gt}(x)}{p(x)}\right)\ dx.
\label{eq:kl}
\end{align}
We assume that $p_{gt}$ is a bivariate normal distribution with standard deviation \mbox{$\sigma_x=\sigma_y=0.05\,m$}, chosen according to the specifications of the Applanix localization system. 
The formula for the pose likelihood $p$ depends on the approach:
\begin{align}
\label{eq:pmci}
p = \begin{cases}
p(x \mid z, m) \propto p(z \mid x, m) &\text{for MLM} \\
p(x \mid z, X_m, Z_m) \propto L(z,x) &\text{for FMP}
\end{cases}
\end{align}

To evaluate Equation~\eqref{eq:kl}, we perform a Monte-Carlo integration over $x$:
\begin{align*}
D_{KL}(p_{gt} \parallel p) \approx \sum \log\left(\frac{p_{gt}}{p}\right),
\end{align*}
where $p_{gt}$ and $p$ are normalized over the samples.
We compute the results shown in Table~\ref{tab:kl} by summing up the approximated $D_{KL}$ values for all measurements in one dataset and dividing the result for MLM by the result of FMP.
All ratios are greater than one, which implies that the distribution estimated by FMP  is closer to the ground truth than the distribution estimated by MLM.
Moreover, all ratios are higher than the corresponding ratios of the measurement likelihoods in Table~\ref{tab:pz}.

\begin{table}
	\normalsize	
	\centering
    \begin{tabular}{ l | r  r }
    		& REF & DEC \\ 
    		\midrule
    campus 	& 1.24 &  1.25\\
    forest 	& 1.27 &   1.82\\
    park 	& 1.37 &  1.59\\
    \end{tabular}
	\caption{Ratios of the Kullback-Leibler divergence from the position distribution estimated by MLM and FMP to the ground truth distribution.
	For each dataset, we show the ratio of the cumulated KL divergences for the MLM approach to the corresponding value for the FMP approach.
	Ratios greater than one indicate that the distribution estimated by FMP is closer to the  ground-truth than the one estimated by MLM.}
	\label{tab:kl}
\end{table}

\subsubsection{Monte-Carlo Localization}

The results of the two aforementioned measures are entirely reproducible, in the sense that they are independent of localization algorithm design.
In the corresponding experiments, the proposed approach always yields better results than the maximum likely approach.
To demonstrate the consequences of this behavior in a real application, we evaluate the localization accuracy of two particle filters.
Each filter employs 3000 particles to localize the off-road robot in six dimensions in the park scenario with the decay-rate model.
The initial variance is 0.1\,m in the translational dimensions and 0.1\,rad in the rotational dimensions.
The filters only differ in the method used to weight the particles: 
One uses the likelihood derived from the most likely map, the other uses the likelihood derived from the full map posterior. 
Figure~\ref{fig:mcl} shows the corresponding localization errors averaged over ten runs.

\begin{figure}[htb]
\centering
\resizebox{\columnwidth}{!}{\input{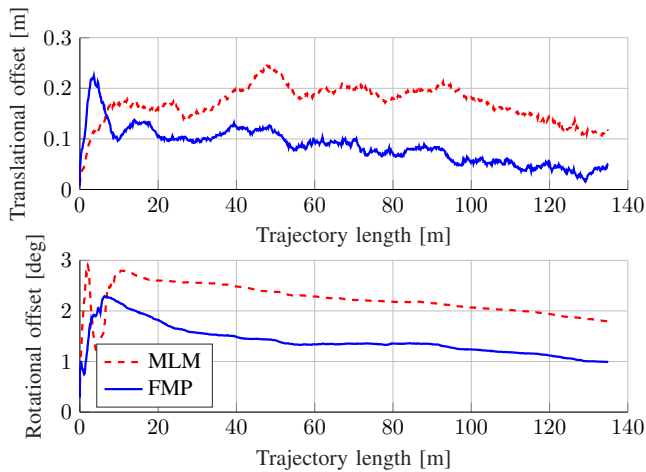}}
\caption{Localization accuracy with the decay-rate model on the park dataset averaged over ten runs.}
\label{fig:mcl}
\end{figure}

\section{Conclusion and Future Work}

In this paper, we present an approach to compute posterior
distributions over real-valued grid maps from lidar observations.  We
demonstrate that for the well-established reflection sensor model and
the recently introduced decay-rate sensor model, the posterior
distributions can be represented in closed form.  Our approach
requires the same measurement information and has the same
computational demands as approaches which only determine the mode of
the distribution.  Simulations and extensive real-world experiments
show that taking into account the full map posterior improves the
accuracy of robot localization.

In the future, we plan to relax the assumption that all map cells are
independent by accounting for measurement noise.  Moreover, we will
embed the presented approach into a SLAM framework, which we will then
use to localize our off-road robot during operation.

\bibliographystyle{plain}
\bibliography{mapposterior}

\end{document}